\documentclass{article}

\usepackage{amsmath,amssymb,amsthm}
\usepackage{xspace}
\usepackage{epsfig}
\usepackage{color}
\usepackage{graphicx}
\usepackage{tikz}
\usepackage{colortbl}
\usepackage{todonotes}
\usepackage[noend]{algorithmic}
\usepackage{color}
\usepackage{url}
\usepackage{booktabs,multirow}
\usepackage[vlined,linesnumbered,ruled,titlenotnumbered]
{algorithm2e}
\newcommand{\ignore}[1]{}

\definecolor{brewerGreen}{RGB}{27,147,108}
\definecolor{brewerOrange}{RGB}{212,85,4}
\definecolor{brewerViolet}{RGB}{105,100,170}

\newtheorem{theorem}{Theorem}

\newtheorem{lemma}[theorem]{Lemma}

\newcommand{\wcity}{\ensuremath{\omega}}
\newcommand{\wtour}{\ensuremath{\mathcal W}}
\newcommand{\lcity}{\ensuremath{\ell}}
\newcommand{\ltour}{\ensuremath{\mathcal L}}
\newcommand{\wdtsp}{{weight\-ed TSP}}
\newcommand{\wdtsps}{W-TSP}

\newcommand{\onewdtsp}{{1-weighted TSP}}
\newcommand{\onewdtsps}{1W-TSP}

\newcommand{\uwdtsp}{{uniform weighted TSP}}
\newcommand{\uwdtsps}{UWDTSP}

\newcommand{\latency}{{Minimum Latency Problem}}
\newcommand{\latencys}{MLP}

\begin{document}

\author{Jakob Bossek$^1$, Katrin Casel$^2$,\\ Pascal Kerschke$^3$, Frank Neumann$^1$ \vspace{5pt} \\  $^1$Optimisation and Logistics, The University of Adelaide,\\ Adelaide, Australia\\
$^2$Chair for Algorithm Engineering,\\ Hasso Plattner Institute, Germany\\
$^3$Department of Information Systems,\\ University of M\"unster, Germany}

\title{The Node Weight Dependent Traveling Salesperson Problem: Approximation Algorithms and Randomized Search Heuristics}

\maketitle

\begin{abstract}
% Several important optimization problems in the area of vehicle routing can be seen as a generalization of the classical Traveling Salesperson Problem (TSP). In this paper, we investigate the relation of such problems where the cost of traveling depends on weights of nodes already visited during a tour. This provides abstractions of important TSP variants such as the Traveling Thief Problem and time dependent TSP variants. We examine the relation of the weight dependent TSP to the classical TSP and provide approximation algorithms for the case of metric distances. Furthermore, we provide experimental insights into how the weights are affecting the quality of the solution by using a state-of-the-art heuristic approach for the TSP and its adaptation to the weight dependent TSP.
%This paper aims to contribute to the theoretical understanding of versions of the %traveling salesperson problem with nodes weights and the performance of evolutionary %algorithms for them.
Several important optimization problems in the area of vehicle routing can be seen as a variant of the classical Traveling Salesperson Problem (TSP). In the area of evolutionary computation, the traveling thief problem (TTP) has gained increasing interest over the last 5 years. In this paper, we investigate the effect of weights on such problems, in the sense that the cost of traveling increases with respect to the weights of nodes already visited during a tour. This provides abstractions of important TSP variants such as the Traveling Thief Problem and time dependent TSP variants, and allows to study precisely the increase in difficulty caused by weight dependence. We provide a 3.59-approximation for this weight dependent version of TSP with metric distances and bounded positive weights.
%Although weight zero on all nodes results in the well-approximable original %metric TSP, weight zero on just some nodes yields instances to which our %approximation is not applicable.
Furthermore, we conduct experimental investigations for simple randomized local search with classical mutation operators and two variants of the state-of-the-art evolutionary algorithm EAX adapted to the weighted TSP. Our results show the impact of the node weights on the position of the nodes in the resulting tour.

% Furthermore, we conduct a series of experiments and gain insights into the difference of quality of solutions calculated by the state-of-the-art evolutionary algorithm EAX and its adaptation to the weighted TSP. Surprisingly, it turns out that performance differences between the vanilla version and the adapted version of EAX decrease with growing instance size.

% Furthermore, we provide experimental insights into how weights are affecting the quality of the solution by using a state-of-the-art heuristic approach for the TSP and its adaptation to the weight dependent TSP.

% \jakob{IMHO können wir das nicht behaupten, da die Experimente es nicht implizieren.}

% (wäre schön wenn wir sowas behaupten könnten) These experimental results indicate that instances containing some nodes of weight zero are indeed the most challenging.

% (falls man keinen Unterschied bemerkt) These experimental results do not indicate a difference in difficulty for the theoretically more challenging case of some nodes of weight zero.

\end{abstract}

% \ignore{
% Outline:
% \begin{itemize}
%     \item Introduction
%     \item Problem formulation
%     \item Computational Complexity and Approximations
%     \begin{itemize}
%         \item Metric TSP (results based on MLP)
%         \item \{1,2\}-TSP based on TSP approximation
%     \end{itemize}
%     \item Analysis of Local Search
%     \item Experimental Investigations of EAX and LKH
%     \begin{itemize}
%         \item Performance difference for 4 combinations of problems/fitness functions
%         \item Morphing of instances from classical TSP to weighted TSP by increasing weights on cities $2, \ldots, n$.
%     \end{itemize}
% \end{itemize}
% }
\section{Introduction}
Evolutionary algorithms have been used for many complex optimisation problems, but it is very hard to understand the complexity of the considered problems as well as the performance of evolutionary algorithms dependent on important problem characteristics.
Complex optimization problems often involve several interacting components that determine the value of a solution. Often these problems are called multi-component problems~\cite{DBLP:books/sp/19/BonyadiM0019} and the goal is to compute an overall high quality solution which might be quite different from good solutions of the underlying silo problems.
The Traveling Thief Problem (TTP) has been introduced in ~\cite{DBLP:conf/cec/BonyadiMB13} as an example problem which combines two of the most well studied $\mathcal{NP}$-hard problems in combinatorial optimization, namely the Traveling Salesperson Problem (TSP) and the Knapsack Problem (KP).
The TTP searches for a TSP tour and a packing plan such that the overall benefit of the tour is maximized. Here, the overall benefit is given by the profit of the collected items minus the cost of the tour which takes into account that costs are increasing with the weight of the collected items. The problem has obtained significant attention in the evolutionary computation literature in recent years.
Different types of evolutionary and other heuristic approaches have been designed for the TTP~\cite{Faulkner2015,ElYafrani:2016:PVS:2908812.2908847,ElYafrani2018231,DBLP:conf/gecco/WuP0N18,DBLP:conf/seal/Wu0PN17} and various types of studies have been carried out to understand the interaction of the two underlying subproblems~\cite{DBLP:conf/gecco/WuPN16,DBLP:journals/soco/MeiLY16,Wagner2017ttpalgssel}. Furthermore, the TTP has been subject to various competitions over the last $5$ years using a benchmark set that combines popular classes of benchmarks for the TSP and KP~\cite{DBLP:conf/gecco/PolyakovskiyB0MN14}.

The goal of this paper is to study such interactions of having increasing weights for the TSP from a theoretical perspective. In the case of TTP, previous studies have focused on the theoretical investigations of the underlying packing problem when the tour is kept fixed. An exact approach based on dynamic programming and a fully polynomial time approximation scheme have been presented in \cite{DBLP:conf/algocloud/NeumannPSSW18}. These studies show that the underlying packing problem (although $\mathcal{NP}$-hard) is relatively easy to solve by these approaches. Dealing with the traveling salesperson part of TTP seems to be the harder problem. The traveling salesperson part of TTP involves traveling costs that are increasing with the weight of the items collected. Motivated by the TTP, we study a variant of the TSP, which we call node weight dependent TSP (\wdtsps), where there are additional weights on the nodes and where the cost of a tour increases with the weight of already visited nodes. Its increased difficulty in comparison to classical TSP shows that TTP is not just more difficult than TSP because of the added knapsack problem. \wdtsps{} highlights the challenges of solving TTP that do not originate from packing decisions. Aside from these investigative purposes, \wdtsps{} is a natural model for many practical applications. For example, it can be used to model a vehicle's loaded weight on its gas consumption, e.g., in case of flights, liner ship movements or waste disposal.

The \wdtsps{} is related to other variants of the TSP. Both definitions of the time dependent TSP (TDTSP) also consider change of costs with respect to the already traveled tour. In \cite{tdtsp_origin}, the authors consider traveling costs that depend on the position in the tour (TDTSP1), a different version (TDTSP2) studied in~\cite{wrong_tdtsp} defines cost dependence with respect to the distance traveled. The very general form of time dependence in both versions results in optimization problems that are very hard to analyze, hence there are few known positive results. Most studies on the TDTSP1 focus on exact algorithms based on integer linear programming, e.g.~\cite{tdtsp_newer_ilp}. For the restriction to distances in $\{1,2\}$,  a $(2-{2}/{3n})$-approximation for the TDTSP1 has been presented in~\cite{1-2-approx}. Furthermore, for TDTSP2 a genetic algorithm approach has been studied in~\cite{tdtsp_genetic}.

Our more restrictive set of weights for \wdtsps{} turns out to be much closer related to the so-called minimum latency problem (MLP). This problem was introduced in~\cite{Blum_latency} to model a cost function for tours from a customer perspective. Formally, the objective is to minimize the average distance from a given start city to all other cities in a TSP tour. At first glance, latency may seem very different from weight dependence, but we will see that these two additional dependencies on the cost of a TSP tour have useful similarities. We pay particular attention to the connection of the classical TSP and the MLP to the \wdtsps{}. Our goal is to examine what types of successful algorithmic approaches for the TSP can be translated to the \wdtsps{}.

\subsection{Our contribution}
We study the \wdtsps{} from a theoretical perspective and are particularly interested in how the problem changes dependent on the weights that are part of the input.
We start by introducing approximation algorithms that make use of known methods for the MLP. Particularly, we investigate the restriction to metric distances and show that if there is an $\alpha$-approximation for MLP then there is a $3\alpha$-approximation for \wdtsps{} restricted to all weights equal to $1$. By adapting the techniques used to approximate the MLP and further structural properties, we derive a 3.59-approximation for metric \wdtsps{} with bounded integer weights.
Afterwards, we investigate the quality of approximations for the \{1,2\}-TSP and show how this translates into a 1.75-approximation for \{1,2\}-\wdtsps{} when all weights are $1$.

Our theoretical investigations are complemented by experimental investigations that systematically investigate the performance of randomized local search using different mutation operators. Furthermore, we investigate the high performing evolutionary algorithm EAX for the classical TSP and its adaptation to the \wdtsps{}. We study the effect of increasing weights for the \wdtsps{} and point out differences that occur in the tours for the TSP and the \wdtsps{} when using EAX on these two problems. For randomized local search, we observe that the inversion operator is preferred over jump operations although even symmetric TSP instances lead to non symmetric instances for \wdtsps{}.
For EAX, the results for $n=50$ cities show that the performance when using the best TSP tour computed by EAX and the best tour for \wdtsps{} might differ by a factor of up to $2.75$ in terms of quality for \wdtsps. For the considered instances having 1000 nodes, we regularly observe a difference by a factor of $1.75$.
%\jakob{Clearly state what we did experimentally.}

The paper is structured as follows. In Section~\ref{sec2}, we formally introduce the weighted Traveling Salesperson Problem. In Section~\ref{sec3}, we provide theoretical approximation guarantees for the metric case and provide improved results for the case where TSP costs are $1$ or $2$ in Section~\ref{sec4}. In Section~\ref{sec5}, we study experimentally different mutation operators for randomized local search as well as the difference of the quality of solutions for the classical TSP and \wdtsp{} obtained by EAX.
Finally, we conclude the paper and point out several promising future research directions.

%\todo{more precise/details here?}
\section{Problem Formulation}
\label{sec2}
We consider the symmetric TSP where cities have additional weights. The cost traveled along an edge depends on the weight of the cities visited so far and the distance of the edge. Let $\pi=(\pi_1, \ldots \pi_n)$ be a permutation of the $n$ cities. We assume that we always start at city $1$ and for the evaluation of a permutation, i.e. we have $\pi_1=1$. If this is not the case, we simply rotate the permutation prior to the fitness evaluation such that city $1$ is the first city in the permutation.

For distance function $d$ and weight function $w$ on a set of $n$ cities, we aim to find a permutation $\pi$ that minimizes the weighted TSP cost, denoted by $\wtour(\pi)$, formally given by the expression
\[
d(\pi_n, \pi_{1})\left(\sum_{j=1}^n w(\pi_j) \right)+ \sum_{i=1}^{n-1}  d(\pi_i, \pi_{i+1})\left(\sum_{j=1}^i w(\pi_j) \right).
\]

We call this optimization problem the node weight dependent TSP (\wdtsps). Note that the standard (unweighted) TSP is the special case where $w(\pi_1)=1$ and $w(\pi_i)=0$, $2 \leq i \leq n$. Generally, the distance $d(\pi_i, \pi_{i+1})$ is multiplied by the weight at city $\pi_i$, which we abbreviate with $\wcity(i)=\sum_{j=1}^i w(\pi_j)$.

To analyze the properties of \wdtsps, we consider the following variants. With \uwdtsp{} (\uwdtsps) we refer to the restriction that each city (except for the fixed start city) has the same weight, formally, $w_i=a$, $2\leq i \leq n$, for some fixed value $a\geq0$. Also, \onewdtsp{}  (\onewdtsps) denotes the further restriction to $a=1$.

This formal definition of the \wdtsps{} and its variations relates to known variations of the TSP as follows.
Time dependence as defined in~\cite{tdtsp_origin}, considers a collection of distance values $d_{i,j,\ell}$, $1\leq i,j\leq n$, $1\leq \ell <n$ with the interpretation that the cost of traveling from city $i$ to city $j$ in a tour where $i$ is the $\ell$th city to be visited is $d_{i,j,\ell}$. \uwdtsps{} with unit weight $a>0$ can be modeled by such a time dependent formulation by setting $d_{i,j,\ell}=d(i,j)(n-\ell+1)a$. For the general \wdtsps{} however, the cost of a transition in the weight dependent TSP does not only depend on its position in the tour, but also on the cities that were previously visited (their respective weights to be precise), hence a form of dependency that can not purely be modeled with in relation to the position in the tour.

 Time dependence as considered in~\cite{tdtsp_origin}, defines for the transition from~$i$ to~$j$, a cost that varies with respect to the time that has passed (i.e. the traveling cost) until the tour reaches city~$i$.
This version of time dependence is similar to our weight dependence in the sense that it also  models distance variation with respect to the partial tour traversed before reaching a city. With time dependence however, distance and time dependence are inherently entangled while weight dependence retains a stronger separation between distance and weight-effects.

\onewdtsps{} has an interesting relationship with the minimum latency problem (MLP) introduced in~\cite{Blum_latency}. We give a formal definition for the MLP and discuss the relevant known results in more detail in the next section.

\section{Approximation Algorithms}\label{sec3}
In this section we consider \wdtsps{} restricted to distances that satisfy the triangle inequality. We call this variant metric \wdtsps{}. Without this restriction, the \wdtsps, like most variants of the TSP, can not be approximated within any constant factor; this immediately follows from the standard reduction from the $\mathcal{NP}$-hard Hamiltonian cycle problem.

Aside from these complexity theoretic reasons, restriction to metric distances is a standard assumption for the TSP. Sometimes, triangle inequality is also indirectly implied by the objective of finding a shortest tour that visits each city at least once,  first introduced by~\cite{graphical_tsp} as \emph{graphical} TSP.

\subsection{Connections to the Minimum Latency Problem}
We explore the connection between the \onewdtsps{} and the MLP. %This allows to give a constant factor approximation for \onewdtsps.
Formally, the \latency{} (also called \emph{delivery-man},  \emph{school-bus driver}, or   \emph{traveling repairman problem}) models the task to find, for a given set of cities with distance function $d$ and a fixed start city $p$, a path starting at $p$ which visits all cities and minimizes the sum of waiting times. Formally, with a solution again modeled as a permutation  $\pi=(\pi_1,\dots,\pi_n)$ with $p=\pi_1$, MLP minimizes
%\jakob{What is $p_i$? I guess you mean $\pi_i$ or simply $i$ in the %following formula.}\frank{changed, Katrin please check that $\lcity(\pi_i)$ %is correct}

\[
\ltour(\pi)=\sum_{i=2}^n \lcity(\pi_i)=\sum_{i=2}^n\sum_{j=1}^{i-1} d(\pi_j,\pi_{j+1})
\]

The shorthand $\lcity(i)$ describes the \emph{latency} of city $\pi_i$  as it models the distance passed until city $\pi_i$ is reached. Although the MLP asks for a path and not a round-trip, it is possible to relate it to \onewdtsps{} in case of metric distances.

For \onewdtsps, the cost of a permutation $\pi=(\pi_1,\dots, \pi_n)$ can be rewritten to
\[
\wtour(\pi)=nd(\pi_n,\pi_1)+\sum_{i=1}^{n-1} id(\pi_i,\pi_{i+1}).
\]
Rewriting the summation to compute the latency of a permutation yields the following connection to \onewdtsps:
\begin{align*}\ltour(\pi)&=\sum_{i=1}^{n-1} (n-i)d(\pi_i,\pi_{i+1}) =\sum_{j=1}^{n-1} jd(\pi_{n-j},\pi_{n-j+1})\\&=\wtour(\pi_n,\pi_{n-1},\dots,\pi_1)-nd(\pi_1,\pi_n)
\end{align*}
Reversing the order of a permutation reveals a strong connection between latency and \wdtsps. As already observed by~\cite{Blum_latency}, an additional shift to start both tours at the fixed start city $\pi_1$ it follows that the cost of a node weighted TSP tour can be interpreted as the sum of a reversed latency tour and a classic TSP tour, formally:
\[\ltour(\pi)=\wtour(\pi_1,\pi_n,\dots,\pi_2)-\sum_{i=1}^{n-1} d(\pi_i,\pi_{i+1}) - d(\pi_1,\pi_n)\]
Since all distances are non-negative, this relation shows that the optimum value for MLP gives a lower bound for the optimum for \onewdtsps. Further, triangle inequality implies $d(\pi_1,\pi_n)\leq \sum_{i=1}^{n-1} d(\pi_i,\pi_{i+1})$, which together with the rough bound  $\ltour(\pi)\geq \sum_{i=1}^{n-1} d(\pi_i,\pi_{i+1})$ yields:
\[\wtour(\pi_1,\pi_n,\dots,\pi_2)\leq \ltour(\pi)+2\sum_{i=1}^{n-1} d(\pi_i,\pi_{i+1})\leq 3 \ltour(\pi) \]
In the worst case, this relation is tight as seen by the example below, where $\ltour(1,2,3)=x+2\varepsilon$ and $\wtour(1,3,2)=x+2x+2\varepsilon$ which for large $x$ and small $\varepsilon$ yields the worst-case factor of~3 between \onewdtsps{} and MLP (observe that the respective permutations are optimum solutions).
\begin{center}
\begin{tikzpicture}
\node[draw, circle, inner sep=2pt] (a) at (0,0) {$1$};
\node[draw, circle, inner sep=2pt] (b) at (1,0) {$2$};
\node[draw, circle, inner sep=2pt] (c) at (5,0) {$3$};
\draw[-] (a) to node[above] {$\varepsilon$} (b);
\draw[-] (b) to node[above] {$x$} (c);
\draw[-,bend right=20] (a) to node[below] {$x$} (c);
\end{tikzpicture}
\end{center}
In general this connection hence only yields that any $\alpha$-approximation for metric MLP can be used to approximate metric \onewdtsps{}with  ratio $3\alpha$. A direct application of the techniques used for \latencys{} allows to derive better approximation results for \onewdtsps.

Over the past 25 years, approximation algorithms for \latencys{} have been gradually improved from the initial~144-approximation in~\cite{Blum_latency} to the currently best 3.59-approximation in~\cite{latency_best_apx}. All such approximations have the same underlying idea of appending a certain set of tours starting and ending at the fixed start city. These tours are approximate solutions to $k$-MST, the problem of finding a minimum cost tree spanning $k$ vertices which is an obvious lower bound on the latency of the $k$-th vertex in an optimal \latencys{} tour. These constructions hence always calculate with the cost of a tour that returns to the start, so they can also be interpreted as a solution to \onewdtsps. The basic idea of our following approximation is to alter the procedure that picks the approximate solutions to the $k$-MST problem according to the objective of \onewdtsps. The formal description with technical details of this idea are given in the proof below.
%\todo{change to "are in appendix" if we run out of space}

%With adjusted costs for the auxiliary graph $H$ to map the weighted tour cost, and a better choice for the last tour that is appended, this technique can be used to show the following.
%Looking closer, some adjustments to the best approximation for \latencys yields an approximation algorithm with the same ratio of~3.59 for \onewdtsps.

\begin{theorem}\label{3-apx}
Metric \onewdtsp{} can be approximated within a ratio of at most~3.59 in polynomial time.
\end{theorem}
\begin{proof}
We adapt the strategy in~\cite{latency_best_apx} as follows. Consider a given metric instance of \onewdtsps{}  with distance $d$ on $n$ cities. First assume that we have tours $T_k$ that are a 2-approximation to the $k$-MST problem, for each $1\leq k\leq n$; which we will also refer to as \emph{good $k$-tours}. As a first difference to the approximation for latency, in this set of good $k$-tours, we construct as $n$-tour a 1.5-approximate solution to TSP, calculated from the algorithm of~\cite{christofides}, to have a certain  approximation ratio for the last tour.\par
As already mentioned, the final solution is constructed by appending a subset of the good $k$-tours. To find a good sequence to build this final solution, we also create an auxiliary graph $H$ that contains a node for each good tour $T_k$ and weights on directed arcs that reflect the cost produced by appending these tours in a solution, and search for a shortest path from $T_1$ to $T_n$. To now reflect the cost of the \onewdtsps{} instead of the \latencys, we change the cost of a path from the node corresponding to $T_i$ to the node corresponding to $T_j$ for any $1\leq i<j\leq n$ in the weighted graph $H$ to $(n-\frac{i+j}{2}+1)c(T_j)$, where $c(T_j)$ denotes the cost of the tour $T_j$. This additional cost of $c(T_j)$ compared to the construction used for the latency problem gives exactly the cost of latency plus TSP tour, hence the cost for the \onewdtsps.\par
%then a shortest path in $H$ corresponds to the weight of the reversed tour, i.e., latency plus TSP-cost. \todo{add some details of the computation here..}
%
Consider, like in the original approach, appending the following set of subtours. For some $c>1$ let $T_{n_i}$ be the good tour of length at most $2bc^i$ that contains the largest number of vertices, with  $b$  set to be $c^U$, for a random variable $U$ uniformly distributed between 0 and 1.
Append these tours $T_{n_i}$ for $i=1,2,\dots$ in this order for all values of $i$ for which $2bc^i$ is strictly smaller than $c(T_n)$, then append $T_n$ as the last tour. Metric shortcuts of this yields a valid solution to \onewdtsps. With $c=3.59$, the bounds on the latency of each city in the resulting tour remains~3.59 with exactly the calculations as presented in \cite{latency_best_apx}. For the additional TSP-cost, we claim that the constructed tour is at most 3.07 times as long as an optimal TSP tour. Let $j$ and $d$ be such that $d\leq 1$ and that the cost of an optimal TSP tour is $dc^j$. Regardless of the value of $b$, the last tour appended by the algorithm is the $n$-path created by  Christofides' algorithm, so it has a cost of at most $1.5dc^j$. Similar to the computations for the latencies, the other appended tours depend on the relation between $d$ and $b$. \par
If $d<b$, the algorithm appends (aside from the $n$-tour), $j-1$ tours, up to cost $2bc^{j-1}$, with combined cost of at most $$2\sum_{\ell=1}^{j-1}bc^\ell < \frac{2bc^j}{c-1}.$$
If $d>b$, the algorithm appends $j$ tours, up to cost $2bc^j$, with combined cost
$$2\sum_{\ell =1}^j bc^\ell<\frac{2bc^{j+1}}{c-1}.$$
With expectation over $U$, the expected length of the tour up to $T_n$ is at most:
\[\int_{\log_c d}^1   \frac{2c^Uc^j}{c-1} \ \mathrm{d}U + \int_0^{\log_c d}\frac{2c^Uc^{j+1}}{c-1} \ \mathrm{d}U=\frac{2dc^j}{\ln c}\]
Overall, the expected ratio between the constructed tour and an optimal TSP tour is at most $1.5+\frac{2}{\ln c}<3.07$.\par
At last, the primal-dual procedure described in~\cite{latency_best_apx} only gives a set of good $k$-tours for a subset of $\{1,\dots,n\}$, not for the whole set as we assumed in the beginning. Exactly as shown for the latency problem, the tours for the missing values of $k$ can be replaced by phantom tours which then are replaced by existing ones since our distance function shows the same behaviour as the original one with respect to the interpolation used for the phantom tours.
\end{proof}

\subsection{Bounded Integer Weights}
So far, we only derived approximation results for the \wdtsps{} for the restriction to all weights equal to~1 with the help of the \latencys. While there exist generalizations of \latencys, there are no known approximation results which translate to \wdtsps.
In~\cite{repairmen}, a variation of the \latencys{} with a service time at each city which adds to the latency of the following city has been investigated. Considering the reverse tour, this is very different from \wdtsps. Weights other than~1 in \wdtsp{} in a reverse tour, can be interpreted as the importance of a city, given as multiplicative factor on the penalty of its waiting cost. To the best of our knowledge, this generalization of the \latencys{} has not been studied.\par
In order to generalize the approximation for \onewdtsps{} to different weights, we exploit structural properties of optimal solutions. This approach yields the following result.
\begin{lemma}
For any $\alpha>1$, an $\alpha$-approximation for metric \onewdtsp{} can be used to derive an $\alpha$-approximation for metric \wdtsp{} with polynomially bounded, non-zero, integer weights.
\end{lemma}
\begin{proof}
Consider an instance of \wdtsps{}  with polynomially bounded non-negative integer weights given by distances $d$ and weights $w$ on $n$ cities. Create an instance of \onewdtsps{} by including  $w(i)$ copies of city $i$, for each city $i\in\{1,\dots,n\}$. Denote for the new instance formally the set of cities by $\{\{i_1,\dots,i_{w(i)}\}\colon 1\leq i\leq n\}$. We further define the distances $\smash{\hat d}$ for the transformed instances by \[\hat d(i_r,j_s)=
\begin{cases} 0, & \text{if } i=j \\
 d(i,j), & \text{else}\end{cases}\]
Observe that this definition yields a metric distance. Further, since the weights are polynomially bounded, this construction is polynomial. Denote by $r$ the number of cities in this new instance, and assign the weight~1 to each of these cities.\par
First observe that any permutation $(\pi_1,\dots, \pi_n)$ for the original instance, can be translated to a permutation of the same weighted cost for the new instance, by replacing $\pi_j$ with $\pi_j=i$ by the sequence $i_1,\dots,i_{w(i)}$. In particular, the optimum value for the new instance is smaller or equal to the optimum value of the original one.\par
Conversely, we can use a permutation of the new instance to create a permutation of the same or even smaller cost for the original instance as follows. Let $(\pi_1,\dots, \pi_r)$ a permutation for the new instance. We claim that this permutation can be altered such that all copies of an original city occur consecutively together which allows to extract a permutation to the original instance by replacing the grouped copies by the single original city. Assume that  for some $1< i\leq n$, the cities $i_1,\dots,i_{w(i)}$ do not occur (in some arbitrary order) as one consecutive block in the sequence $(\pi_1,\dots, \pi_r)$. Let $1\leq x\leq w(i)$ be such that $i_x$ occurs last, among all cities in $\{i_1,\dots,i_{w(i)}\}$, in the sequence $(\pi_1,\dots, \pi_r)$. Consider altering the sequence, by moving all cities in $\{i_1,\dots,i_{w(i)}\}\setminus\{i_x\}$ to be visited right after $i_x$. All edges after $i_x$ have the exact same cost, since neither the weight nor the cities have changed. All edges among the cities in $\{i_1,\dots,i_{w(i)}\}$ are zero, so they add no cost at all. All edges before the block $\{i_1,\dots,i_{w(i)}\}$ now are attached with equal or less weight than before, since the weight of the shifted cities is postponed. Triangle inequality allows to estimate the new transitions created by moving  by the cost of the previously two edges traveling through cities in $\{i_1,\dots,i_{w(i)}\}$. Repeating this procedure yields a permutation that can be translated to the original instance and has the same or smaller cost.\par
Overall, it follows that an $\alpha$-approximate solution for metric \wdtsps{} with bounded non-negative integer weights can be computed by creating the new instance of metric \onewdtsps{}, running the assumed $\alpha$-approximation on it, and then translating the resulting permutation to the original instance (this can be done in linear time by simply scanning the permutation in reverse and skipping duplicates).
%So there exist $i_x$ and $i_y$ with $1\leq x,y\leq w(i)$ such that $\pi_x$
\end{proof}
Combined with Theorem~\ref{3-apx}, this result gives the following.
\begin{theorem}
Metric \wdtsp{} with polynomially bounded, non-zero, integer weights can be approximated within a ratio of at most~3.59 in polynomial time.
\end{theorem}

\section{1-Weighted TSP\{1,2\}}
\label{sec4}
We now consider the further restriction to distance values 1 and 2. For the classical TSP, this is one of the most studied restrictions, usually called  \{1,2\}-TSP, as this problem can be seen as a generalization of the Hamiltonian cycle problem and is therefore still $\mathcal{NP}$-hard. Different approximation algorithms have been developed for the \{1,2\}-TSP and we investigate how to make use of those when investigating \onewdtsps{} with distances 1 and 2. We refer to  this restriction by \onewdtsp{\{1,2\}}, \onewdtsps{\{1,2\}} for short.

A $(2-{2}/{3n})$-approximation algorithm for a restriction to distances 1 and 2 on the related time dependent TSP has been presented in~\cite{1-2-approx}. Although \onewdtsps{} is a special case of TDTSP1, this result can not be used to derive an equivalent approximation for \onewdtsps\{1,2\}, since edge-cost restriction for our problem does not translate to edge-cost restriction to 1 and 2 in the representation as TDTSP1; observe that the costs of 1 and 2 have to be multiplied by the weights, which, even with the restriction to all weights being~1, gives a range of time dependent distances between 1 and $2n$.

%\subsection{Approximation Algorithms based on Cycle Covers}
We first consider the case where the input allows for a TSP tour of cost $n$.
%Then a cycle cover of cost $n$ can be obtained using $2$-matching %algorithms.
%In order to deal with this case, it is crucial to see how many 2-edges are introduced at most into a tour and where they are located.
%We first consider the restriction .
% case where all weights are equal and w.l.o.g. we assume $w(i)=1$, $1 \leq i \leq n$.
First observe that for this case, an optimal tour for \onewdtsps{}  has cost
$\sum_{i=1}^n = n(n+1)/2.$
Let $k$ be the number of $2$-edges introduced into the tour by an approximate solution. The tour has the highest possible costs if these edges are at the end of the tour.
Compared to the optimal tour the cost increase by
\begin{eqnarray*}
& &n+(n-1)+ \ldots + (n-(k-1))\\
&=& kn -k(k-1)/2 = k(n-(k-1)/2)
\end{eqnarray*}

Let $\pi$ be an $\alpha=(1+c)$-approximation, $c\geq 0$ for the \{1,2\}-TSP. Let $k\leq cn$ be the number of  $2$-edges in $\pi$.
%(Note that our result holds for any tour that contains at most $cn$ $2$-edges as the value of an optimal solution can only be higher if it constraints at least one $2$-edge.)
The resulting approximation ratio for \onewdtsps{\{1,2\}} is at most
\begin{eqnarray*}
& & 1+ k(n-(k-1)/2)/ (n(n+1)/2)\\
&\leq & 1+ (2kn - (k^2-k)/2)/(n^2)
\end{eqnarray*}
Setting $k=cn$, we get
\begin{eqnarray*}
& & 1+ (2cn^2 - (c^2n^2-cn)/2)/(n^2)\\
&= & 1+(2c-(c^2-c/n)/2)\\
& = & 1+(2c-c^2/2) + o(1)
\end{eqnarray*}
Assume that we use the $7/6$-approximation for \{1,2\}-TSP given in \cite{DBLP:journals/mor/PapadimitriouY93}, then we have $c=1/6$ and therefore a $1+(2/6 - 1/72)+o(1) = 95/72 +o(1)$ approximation for the weight dependent TSP.

\ignore{
We can improve the result if we assume that all cycles in the two matching have length at least $r$. Then only every $r$-th edge in the TSP tour is of cost $2$ which changes the cost increase to

\begin{eqnarray*}
& & n+(n-r)+ \ldots + (n-r(k-1))\\
& = & kn -rk(k-1)/2 = k(n-r(k-1)/2)
%\leq (kn - rk^2/2)/(n^2)
\end{eqnarray*}

Consequently, the approximation ratio when using a tour that has at most $k=cn$ $2$-edges that appear only every $r$-th time in the tour is
\[1+ (cn^2 - rc^2n^2/2 +kr/2)/(n^2)= 1+(2c-rc^2/2) +o(1)\]
Using again $c=1/6$ and $r=3$, we get an approximation of $1+(2/6 - 3/72)+o(1)=93/72+o(1)$.

We can improve our results by considering the produced tour in addition in reverse order.
Let $\pi=(\pi_1, \ldots, \pi_n)$ be a tour with $k$ $2$-edges. We also consider the tour $\pi'=(\pi_1, \pi_n, \ldots, \pi_2)$. One of these two tours has at least $k/2$ edges of cost $2$ at positions $1, \ldots, n/2$. This implies that the cost of the weighted TSP tour is at most
\[n + kn - kn/2 - \sum_{i=1}^{k/2-1} i = kn/2 - k^2/4 + k/2\]
Hence, the approximation ratio is  upper bounded by
\[1+(cn^2/2 - (cn)^2/4 + cn/2)/(n^2/2) = 1+c - c^2/2 +o(1)\]
which gives an approximation of $1+ 1/6 - 1/72 +o(1) = 83/72+o(1)$ for $c=1/6$.
}
We can improve our results by considering the $\pi$ also in reverse order.
Formally, for $\pi=(\pi_1, \ldots, \pi_n)$ we also consider the tour $\pi'=(\pi_1, \pi_n, \ldots, \pi_2)$ .One of these two tours has at least $k/2$ edges of cost $2$ at positions $1, \ldots, n/2$ which gives, for the better of these tours, an addition to the optimum of at most
\[ kn/2 + kn/4- 2\sum_{i=1}^{k/2-1} i = kn/2 + kn/4- k^2/4 + k/2\]
With $k=cn$ the approximation ratio is hence bounded by
\begin{eqnarray*}
&&1+(cn^2/2 + cn^2/4 - (cn)^2/4 + cn/2)/(n^2/2)\\
&=& 1+1.5c - c^2/2 +o(1)
\end{eqnarray*}
For  $c=1/6$, this gives a ratio of $89/72+o(1)$.

%\subsection{General \{1,2\}-TSP}
We now extend these observations to the general case where the optimal solution can include edges of cost $2$.
Let $\pi^*$ be an optimal solution for the classical \{1,2\}-TSP of cost $OPT=n+u$. $\pi^*$ has exactly $u$ edges of cost $2$.
An $\alpha=(1+c)$-approximation algorithm for the \{1,2\}-TSP produces a tour $\pi$ of TSP-cost at most $(1+c)(n+u) = n + cn + (1+c)u$ which has at most $k=cn + (1+c)u$ edges of cost $2$. Note $k \leq n$.

A lower bound on the value of an optimal solution for the \onewdtsps{\{1,2\}} is obtained by assuming that the $u$ edges of the optimal TSP tour appear at the beginning of the weighted TSP tour. Hence, we can bound   the value of an optimal solution of \onewdtsps{\{1,2\}} by
\[\sum_{i=1}^{n} i + \sum_{i=1}^{u}i = n(n+1)/2 + u(u+1)/2\]
%\textcolor{red}{This is the new bound on the approximate tour}
We now estimate the weighted tour value of the approximate tour $\pi$ and its reversal $\pi'$ more precisely, recall that both tours contain at most $k=cn + (1+c)u\leq n$ edges of cost~2.
%Considering the tour $\pi$  and its reversal $\pi'$, we use the following relation.
Each edge of cost~2 that adds an addition of cost $n\leq r\leq 1$ to $\pi$ (addition compared to $n(n+1)/2$), adds a cost of $n-r+1$ to $\pi'$. Summing up, if all costs of edges of length~2 in $\pi$ cause an addition of $R$, then these edges produce an additional cost of $k(n+1)-R$ for $\pi'$. In the worst case, $R$ is equal to $k(n+1)/2$, which results in a worst-case cost of $(n(n+1)+k(n+1))/2$
for the better of the two options.

Compared to the above bound on the optimal solution, this results in an approximation ratio of at most:

%\textcolor{red}{Here ends this brief argument!!!}

%We assume that the $\alpha$-approximation places its $k=cn + (1+c)u\leq n$ edges at the end of the tour which gives an upper bound

%\begin{eqnarray*}
%\wtour(\pi^*) &  \leq & \sum_{i=1}^{n} i + \sum_{i=n-(k-1)}^{n}i\\
%& = & n(n+1)/2+ n + (n-1) - \ldots - (n-(k-1))\\
%& = & n(n+1)/2+ kn - k(k-1)/2\\
%& = & n(n+1)/2+ k(n-(k-1)/2)
%\end{eqnarray*}
%for the \onewdtsp{\{1,2\}}.
%Approximation ratio is therefore

\begin{eqnarray*}
%& & \frac{n(n+1)/2+ k(n-(k-1)/2)}{(n(n+1)/2 + u(u+1)/2)}\\
%&= &  \frac{n(n+1)/2+ (cn + (1+c)u)(n-((cn + (1+c)u)-1)/2)} {(n(n+1)/2 + u(u+1)/2)}\\
%&\leq & \frac{n(n+1)/2+ (cn + (1+c)u)n} {(n(n+1)/2 + u(u+1)/2)}\\
%&\leq & \frac{n(n+1)/2+ (cn +u+cu)n} {(n(n+1)/2 + u(u+1)/2)}\\
%& \leq & 1+ 2c + (1+c)un/(n^2+u^2)\\
%& \leq & \frac{n(n+1)/2+ kn/2}{(n(n+1)/2 + u(u+1)/2)}\\
%&\leq  & \frac{n(n+1)/2+ (cn + (1+c)u)n/2}{(n(n+1)/2 + %u(u+1)/2)}\\
& &\frac{n(n+1)+k(n+1)}{n(n+1) + u(u+1)}\\
&=  & \frac{n(n+1)+ (cn + (1+c)u)(n+1)}{n(n+1) + u(u+1)}\\
&=  & \frac{n(n+1)+ cn^2 +cn + (1+c)u + (1+c)un}{n(n+1) + u(u+1)}\\
& \leq & 1 + c + (1+c)/2 \\
& \leq & 1.5 \alpha
\end{eqnarray*}

where the last step uses that $\frac{un}{(n(n+1) + u(u+1))}$ is monotonically increasing in $u$ and attains its maximum for $u=n$. We summarize these results in the following theorem.

\begin{theorem}
Using an $\alpha$-approximation algorithm for \{1,2\}-TSP to compute a TSP tour $\pi$, $\pi$ or its reverse tour $\pi'$ is a $1.5 \alpha$-approximation for the \onewdtsp{\{1,2\}}.
\end{theorem}

Using the $(7/6)$-approximation for the \{1,2\}-TSP, we get a $1.75$-approximation for the \onewdtsps{\{1,2\}}.

\section{Randomized Search Heuristics}
\label{sec5}
In this section, we consider randomized search heuristics for \wdtsps{}. We start by investigating variants of randomized local search and examine the use of popular mutation operators traditionally used for the classical TSP. Afterwards, we examine EAX as a state-of-the-art solver for the TSP and its adaptation to \wdtsp{}.

\subsection{Problem instances}

\begin{figure*}[htb]
    \centering
    \includegraphics[width=\textwidth]{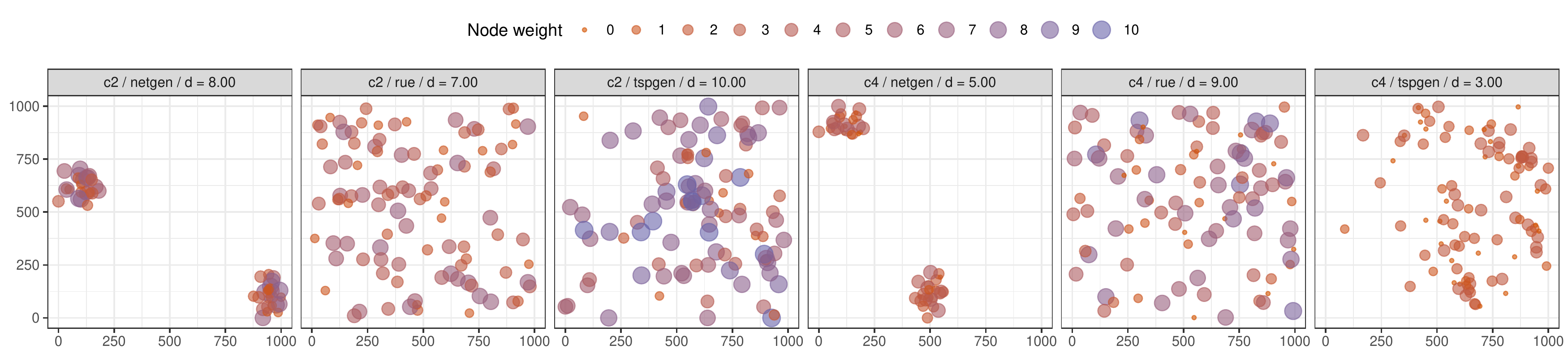}
    \caption{Examples of generated problem instances. Point size and color change with increasing node weight.}
    \label{fig:instance_gallery}
\end{figure*}

We consider a rich set of artificially generated metric TSP instances with different integer node weights.
The instance generation approach is performed in two steps. First, $n \in \{25, 50, 100, 500, 1\,000\}$ nodes are placed in the Euclidean plane (bounded to $[0, 1\,000]^2$) utilizing different node placement generators. In this study we consider Random Uniform Euclidean (\emph{rue}) placement, i.e., node coordinates are sampled uniformly at random within the bounding box.
Moreover, we consider so-called \emph{netgen} placement with two distinct clusters of points sampled from a bivariate Gaussian distribution around the cluster centers. Further, we consider \emph{tspgen} placement. Here, points are initially placed according to rue placement and subsequently altered in an iterative manner by a sequence of mutation operations~\cite{BKNW+19}.
The second step deals with the assignment of node weights. Here we consider three different configurations as described in the following. Note that $w_1 = 1$ across all cases for the fixed start node $p = 1$, whereas three different configurations are considered for $w_i$ (with $2 \leq i \leq n$):

\begin{itemize}
    \item $C_1$: $w_i = d$ with $d \in \{0.0, 0.1, \ldots, 1.0\}$,
    \item $C_2$: $w_i \in \{1, \ldots, d\}$ with $d \in \{2, \ldots, 10\}$ and
    \item $C_3$: $w_i \in \{0, \ldots, d\}$ with $d \in \{1, 2, \ldots, 10\}$.
\end{itemize}

To account for randomness in node placement and weight assignment we generate ten instances for each combination, which add up to $45\,000$ instances\footnote{The total number of instances results from 30 different configurations (see details of $C_1$ to $C_3$), five instance sizes ($n$), ten replications due to node placement and ten replications for the weight-to-node assignments.} (see Fig.~\ref{fig:instance_gallery} for examples).

\subsection{Performance of the RLS variants}

\begin{algorithm}[t]
     Choose a permutation $\pi$ of the given $n$ cities uniformly at random.\\
Produce $\pi'$ from $\pi$ by mutation.\\
    If $\wtour(\pi') \leq \wtour(\pi)$, set $\pi:=\pi'$.\\
   If not termination condition, go to 2).
  \caption{Randomized Local Search (RLS)}
  \label{alg:rls}
\end{algorithm}

We first consider randomized local search (RLS) shown in Algorithm~\ref{alg:rls}. It starts with a permutation $\pi$ of the given $n$ cities chosen uniformly at random. In each iteration a new permutation $\pi'$ is produced from the current permutation $\pi$ by a simple mutation. The new permutation $\pi'$ replaces $\pi$ if its weighted tour length is not larger than the one of~$\pi$.
We investigate popular mutation operators for the classical TSP in this context, namely inversion, exchange, and jump operations. Inversion operators usually achieve a high performance for the classical symmetric TSP as it only results in the update of the cost of two edges in the cost function. For \wdtsp{} the situation is different as weighted tours are not symmetric and the question arises whether inversion is still a good operator when considering \wdtsp{}.
We run RLS with the three aforementioned mutation operators. Per instance, 30 independent runs are performed with a stopping condition of $1\,000 \cdot n$ function evaluations.
% \katrin{Dieser Satz ist zu lang für mich mit den vielen ``each'', weiss nicht wie ich ihn umformulieren soll...}\pascal{stimme zu; wollte ihn grad umformulieren, aber bin dann ueber die 30 independent runs gestolpert, da doch zuvor schon gesagt wurde, dass man jede mutationsvariante auf jeder RUE-instanz 30x laufen laesst}
The performance is measured as follows: let $\pi$ be the final solution of algorithm $A$ on instance $I$ and let $\pi^{*}$ be the best, i.e., shortest, tour found in all runs of all algorithms on $I$. We measure the performance as the relative deviation from $\pi^{*}$, i.e.,
\begin{align}\label{eq:perf_rls}
\text{perf}(\pi) = \left(\frac{\wtour(\pi)}{\wtour(\pi^{*})} - 1\right)\cdot100 \geq 0.
\end{align}
Note that this value is $0$ if $\wtour(\pi) = \wtour(\pi^{*})$. This measure allows to aggregate over instance sizes.

The results of our experiments comparing the three operators on the three instance classes $C_1$, $C_2$, and $C_3$ are shown in Table~\ref{tab:rls}. Here we report the mean, standard deviation, the median and results of pairwise Wilcoxon-Mann-Whitney tests with Bonferroni $p$-value adjustment of the performance values defined in Eq.~\ref{eq:perf_rls} split by class and $d$-value.
It can be observed that the RLS variant using the inversion operation outperforms the other two variants for almost all settings. Comparing RLS using exchange operations with RLS using jump operations, we can see that jump operations are preferable over exchanges for the class $C_1$ whereas exchanges are preferable over jumps for the classes $C_2$ and $C_3$.
%\katrin{Das sollte $C_2$ und $C_3$ sein, oder?}\pascal{agree :-)}

\begin{table*}[htb]
\renewcommand{\arraystretch}{1.1}
\renewcommand{\tabcolsep}{4.2pt}
\caption{\label{tab:rls}Tabular values of \textbf{mean}, standard deviation (\textbf{std}), \textbf{median} and results of pairwise statistical tests (\textbf{stat}). Results are split by instance classes $C_1$, $C_2$, $C_3$ and the value of $d$.}
\begin{tiny}
\centering
\begin{tabular}[t]{lrrrrrrrrrrrrr}
\toprule
\multicolumn{1}{c}{\textbf{ }} & \multicolumn{1}{c}{\textbf{ }} & \multicolumn{4}{c}{\textbf{\textcolor{brewerGreen}{RLS[Exchange] (1)}}} & \multicolumn{4}{c}{\textbf{\textcolor{brewerOrange}{RLS[Inversion] (2)}}} & \multicolumn{4}{c}{\textbf{\textcolor{brewerViolet}{RLS[Jump] (3)}}} \\
\cmidrule(l{3pt}r{3pt}){3-6} \cmidrule(l{3pt}r{3pt}){7-10} \cmidrule(l{3pt}r{3pt}){11-14}
\textbf{Cl} & $d$ & \textbf{mean} & \textbf{std} & \textbf{med} & \textbf{stat} & \textbf{mean} & \textbf{std} & \textbf{med} & \textbf{stat} & \textbf{mean} & \textbf{std} & \textbf{med} & \textbf{stat}\\
\midrule
 & 0.0 & 104.06 & 76.33 & 67.01 &  & \cellcolor{gray!20}{\textbf{4.35}} & 3.46 & 3.45 & \textcolor{brewerGreen}{$\text{1}^{+}$}, \textcolor{brewerViolet}{$\text{3}^{+}$} & 42.93 & 18.75 & 43.64 & \textcolor{brewerGreen}{$\text{1}^{+}$}\\

 & 0.1 & 65.07 & 33.61 & 56.57 &  & \cellcolor{gray!20}{\textbf{16.80}} & 11.96 & 14.87 & \textcolor{brewerGreen}{$\text{1}^{+}$}, \textcolor{brewerViolet}{$\text{3}^{+}$} & 45.46 & 25.87 & 41.13 & \textcolor{brewerGreen}{$\text{1}^{+}$}\\

 & 0.2 & 64.11 & 33.34 & 55.08 &  & \cellcolor{gray!20}{\textbf{17.98}} & 12.70 & 16.07 & \textcolor{brewerGreen}{$\text{1}^{+}$}, \textcolor{brewerViolet}{$\text{3}^{+}$} & 49.37 & 28.32 & 45.35 & \textcolor{brewerGreen}{$\text{1}^{+}$}\\

 & 0.3 & 64.02 & 33.42 & 55.59 &  & \cellcolor{gray!20}{\textbf{18.71}} & 13.09 & 16.82 & \textcolor{brewerGreen}{$\text{1}^{+}$}, \textcolor{brewerViolet}{$\text{3}^{+}$} & 51.05 & 29.59 & 47.26 & \textcolor{brewerGreen}{$\text{1}^{+}$}\\

 & 0.4 & 64.14 & 33.59 & 56.07 &  & \cellcolor{gray!20}{\textbf{18.96}} & 13.29 & 16.93 & \textcolor{brewerGreen}{$\text{1}^{+}$}, \textcolor{brewerViolet}{$\text{3}^{+}$} & 52.20 & 29.96 & 48.74 & \textcolor{brewerGreen}{$\text{1}^{+}$}\\

 & 0.5 & 63.63 & 33.31 & 54.75 &  & \cellcolor{gray!20}{\textbf{19.28}} & 13.64 & 17.30 & \textcolor{brewerGreen}{$\text{1}^{+}$}, \textcolor{brewerViolet}{$\text{3}^{+}$} & 52.53 & 30.19 & 49.61 & \textcolor{brewerGreen}{$\text{1}^{+}$}\\

 & 0.6 & 63.35 & 33.37 & 54.87 &  & \cellcolor{gray!20}{\textbf{19.23}} & 13.58 & 17.27 & \textcolor{brewerGreen}{$\text{1}^{+}$}, \textcolor{brewerViolet}{$\text{3}^{+}$} & 53.06 & 30.34 & 49.91 & \textcolor{brewerGreen}{$\text{1}^{+}$}\\

 & 0.7 & 63.51 & 33.42 & 54.84 &  & \cellcolor{gray!20}{\textbf{19.36}} & 13.65 & 17.24 & \textcolor{brewerGreen}{$\text{1}^{+}$}, \textcolor{brewerViolet}{$\text{3}^{+}$} & 53.49 & 30.67 & 50.86 & \textcolor{brewerGreen}{$\text{1}^{+}$}\\

 & 0.8 & 63.57 & 33.76 & 54.33 &  & \cellcolor{gray!20}{\textbf{19.69}} & 13.78 & 17.75 & \textcolor{brewerGreen}{$\text{1}^{+}$}, \textcolor{brewerViolet}{$\text{3}^{+}$} & 53.59 & 30.69 & 50.69 & \textcolor{brewerGreen}{$\text{1}^{+}$}\\

 & 0.9 & 63.56 & 33.51 & 54.66 &  & \cellcolor{gray!20}{\textbf{19.60}} & 13.81 & 17.47 & \textcolor{brewerGreen}{$\text{1}^{+}$}, \textcolor{brewerViolet}{$\text{3}^{+}$} & 53.96 & 30.88 & 51.21 & \textcolor{brewerGreen}{$\text{1}^{+}$}\\

\multirow{-11}{*}{\raggedright\arraybackslash $C_1$} & 1.0 & 63.80 & 33.49 & 55.59 &  & \cellcolor{gray!20}{\textbf{19.87}} & 13.91 & 17.89 & \textcolor{brewerGreen}{$\text{1}^{+}$}, \textcolor{brewerViolet}{$\text{3}^{+}$} & 54.12 & 31.01 & 51.69 & \textcolor{brewerGreen}{$\text{1}^{+}$}\\
\cmidrule{1-14}
 & 2.0 & 61.57 & 31.99 & 54.80 &  & \cellcolor{gray!20}{\textbf{21.90}} & 15.07 & 19.86 & \textcolor{brewerGreen}{$\text{1}^{+}$}, \textcolor{brewerViolet}{$\text{3}^{+}$} & 62.10 & 36.99 & 57.16 & \\

 & 3.0 & 59.51 & 29.71 & 53.82 & \textcolor{brewerViolet}{$\text{3}^{+}$} & \cellcolor{gray!20}{\textbf{23.24}} & 15.91 & 21.60 & \textcolor{brewerGreen}{$\text{1}^{+}$}, \textcolor{brewerViolet}{$\text{3}^{+}$} & 64.72 & 37.78 & 60.14 & \\

 & 4.0 & 58.32 & 29.56 & 52.09 & \textcolor{brewerViolet}{$\text{3}^{+}$} & \cellcolor{gray!20}{\textbf{23.88}} & 16.18 & 22.25 & \textcolor{brewerGreen}{$\text{1}^{+}$}, \textcolor{brewerViolet}{$\text{3}^{+}$} & 66.73 & 38.96 & 62.06 & \\

 & 5.0 & 57.19 & 28.54 & 52.56 & \textcolor{brewerViolet}{$\text{3}^{+}$} & \cellcolor{gray!20}{\textbf{24.41}} & 16.63 & 22.86 & \textcolor{brewerGreen}{$\text{1}^{+}$}, \textcolor{brewerViolet}{$\text{3}^{+}$} & 67.13 & 38.33 & 62.97 & \\

 & 6.0 & 56.65 & 28.39 & 51.93 & \textcolor{brewerViolet}{$\text{3}^{+}$} & \cellcolor{gray!20}{\textbf{24.73}} & 16.69 & 23.52 & \textcolor{brewerGreen}{$\text{1}^{+}$}, \textcolor{brewerViolet}{$\text{3}^{+}$} & 67.95 & 39.00 & 63.89 & \\

 & 7.0 & 55.71 & 28.24 & 51.30 & \textcolor{brewerViolet}{$\text{3}^{+}$} & \cellcolor{gray!20}{\textbf{25.08}} & 16.92 & 23.86 & \textcolor{brewerGreen}{$\text{1}^{+}$}, \textcolor{brewerViolet}{$\text{3}^{+}$} & 67.52 & 39.08 & 63.26 & \\

 & 8.0 & 55.26 & 27.69 & 50.93 & \textcolor{brewerViolet}{$\text{3}^{+}$} & \cellcolor{gray!20}{\textbf{25.09}} & 16.86 & 23.89 & \textcolor{brewerGreen}{$\text{1}^{+}$}, \textcolor{brewerViolet}{$\text{3}^{+}$} & 68.20 & 38.69 & 64.59 & \\

 & 9.0 & 55.07 & 27.59 & 50.13 & \textcolor{brewerViolet}{$\text{3}^{+}$} & \cellcolor{gray!20}{\textbf{25.56}} & 17.07 & 24.41 & \textcolor{brewerGreen}{$\text{1}^{+}$}, \textcolor{brewerViolet}{$\text{3}^{+}$} & 68.21 & 38.87 & 63.55 & \\

\multirow{-9}{*}{\raggedright\arraybackslash $C_2$} & 10.0 & 54.56 & 27.10 & 50.02 & \textcolor{brewerViolet}{$\text{3}^{+}$} & \cellcolor{gray!20}{\textbf{25.18}} & 17.16 & 23.90 & \textcolor{brewerGreen}{$\text{1}^{+}$}, \textcolor{brewerViolet}{$\text{3}^{+}$} & 68.40 & 38.82 & 64.13 & \\
\cmidrule{1-14}
 & 1.0 & \cellcolor{gray!20}{\textbf{39.04}} & 19.34 & 39.78 & \textcolor{brewerOrange}{$\text{2}^{+}$}, \textcolor{brewerViolet}{$\text{3}^{+}$} & 41.42 & 26.42 & 43.21 & \textcolor{brewerViolet}{$\text{3}^{+}$} & 62.55 & 32.41 & 62.50 & \\

 & 2.0 & 44.24 & 21.58 & 43.80 & \textcolor{brewerViolet}{$\text{3}^{+}$} & \cellcolor{gray!20}{\textbf{34.33}} & 21.70 & 35.35 & \textcolor{brewerGreen}{$\text{1}^{+}$}, \textcolor{brewerViolet}{$\text{3}^{+}$} & 65.49 & 34.24 & 64.71 & \\

 & 3.0 & 46.26 & 22.64 & 45.26 & \textcolor{brewerViolet}{$\text{3}^{+}$} & \cellcolor{gray!20}{\textbf{32.80}} & 21.06 & 33.07 & \textcolor{brewerGreen}{$\text{1}^{+}$}, \textcolor{brewerViolet}{$\text{3}^{+}$} & 67.80 & 34.84 & 66.98 & \\

 & 4.0 & 47.58 & 23.46 & 46.25 & \textcolor{brewerViolet}{$\text{3}^{+}$} & \cellcolor{gray!20}{\textbf{31.08}} & 19.97 & 31.53 & \textcolor{brewerGreen}{$\text{1}^{+}$}, \textcolor{brewerViolet}{$\text{3}^{+}$} & 68.87 & 36.08 & 67.77 & \\

 & 5.0 & 48.32 & 23.69 & 46.79 & \textcolor{brewerViolet}{$\text{3}^{+}$} & \cellcolor{gray!20}{\textbf{30.02}} & 19.39 & 30.05 & \textcolor{brewerGreen}{$\text{1}^{+}$}, \textcolor{brewerViolet}{$\text{3}^{+}$} & 69.28 & 36.56 & 67.93 & \\

 & 6.0 & 48.55 & 24.43 & 46.12 & \textcolor{brewerViolet}{$\text{3}^{+}$} & \cellcolor{gray!20}{\textbf{28.90}} & 19.06 & 28.40 & \textcolor{brewerGreen}{$\text{1}^{+}$}, \textcolor{brewerViolet}{$\text{3}^{+}$} & 69.27 & 37.44 & 67.23 & \\

 & 7.0 & 49.51 & 24.38 & 47.72 & \textcolor{brewerViolet}{$\text{3}^{+}$} & \cellcolor{gray!20}{\textbf{28.70}} & 18.80 & 28.42 & \textcolor{brewerGreen}{$\text{1}^{+}$}, \textcolor{brewerViolet}{$\text{3}^{+}$} & 69.94 & 37.59 & 68.83 & \\

 & 8.0 & 49.93 & 24.34 & 48.29 & \textcolor{brewerViolet}{$\text{3}^{+}$} & \cellcolor{gray!20}{\textbf{29.07}} & 19.02 & 28.85 & \textcolor{brewerGreen}{$\text{1}^{+}$}, \textcolor{brewerViolet}{$\text{3}^{+}$} & 70.23 & 37.19 & 68.23 & \\

 & 9.0 & 50.45 & 24.28 & 48.58 & \textcolor{brewerViolet}{$\text{3}^{+}$} & \cellcolor{gray!20}{\textbf{28.96}} & 19.30 & 28.42 & \textcolor{brewerGreen}{$\text{1}^{+}$}, \textcolor{brewerViolet}{$\text{3}^{+}$} & 70.60 & 37.48 & 68.79 & \\

\multirow{-10}{*}{\raggedright\arraybackslash $C_3$} & 10.0 & 50.09 & 24.68 & 48.16 & \textcolor{brewerViolet}{$\text{3}^{+}$} & \cellcolor{gray!20}{\textbf{28.50}} & 18.73 & 28.17 & \textcolor{brewerGreen}{$\text{1}^{+}$}, \textcolor{brewerViolet}{$\text{3}^{+}$} & 70.22 & 37.74 & 68.05 & \\
\bottomrule
\end{tabular}
\end{tiny}
\end{table*}

\subsection{Performance of EAX}

Next we investigate the adaptation of the evolutionary TSP solver EAX~\cite{Nagata2013}. EAX is an evolutionary algorithm which uses a powerful \emph{edge assembly crossover} operator to produce high-quality offspring individuals and a sophisticated population diversity mechanism.
This algorithm has shown state-of-the-art performance in inexact TSP-solving in various studies~\cite{Nagata2013, Kotthoff2015, KKBHTLeveragingTSP}. We modified the algorithm to enable handling of node weights and consider two different fitness functions that guide the evolutionary search process: the classical TSP fitness function (ignoring node weights) and a fitness function based on the weighted TSP costs $\wtour(\pi)$.
In the following we use the abbreviations EAX and W-EAX for brevity. Our main interest is the difference of tour lengths obtained by runs of EAX and W-EAX, respectively, depending on the structure of the \wdtsp{} instances under consideration.

We then performed ten independent runs on each instance with both fitness functions, resulting in a total of $900\,000$ experiments which were strongly parallelized on a high performance computing cluster. EAX was run with a time-limit of five~seconds for instances with up to 100 nodes and three~minutes for larger instances to keep the computational costs reasonable. These values may seem small at first glance, however, studies in \cite{KKBHTLeveragingTSP} revealed that EAX is able to solve even large TSP instances -- with thousands of nodes -- to optimality within few seconds. Note that after completion of each run -- regardless of the fitness function used internally as a driver -- the final tour was evaluated by means of the weighted TSP fitness function.
For evaluation we calculate the weighted tour length ratios, i.e., the weighted tour length obtained by EAX divided by the respective weighted tour length obtained by W-EAX for each instance and run. Note that values greater than~1 indicate an advantage of W-EAX over EAX. Fig.~\ref{fig:boxplot_ratios} shows the distributions of weighted tour length ratios separated by instance size, configuration and maximum weight $d$.
%We omit results for other values of $n$ since they do not provide any additional insights.

\begin{figure*}[tb]
    \centering
    \includegraphics[width=\textwidth]{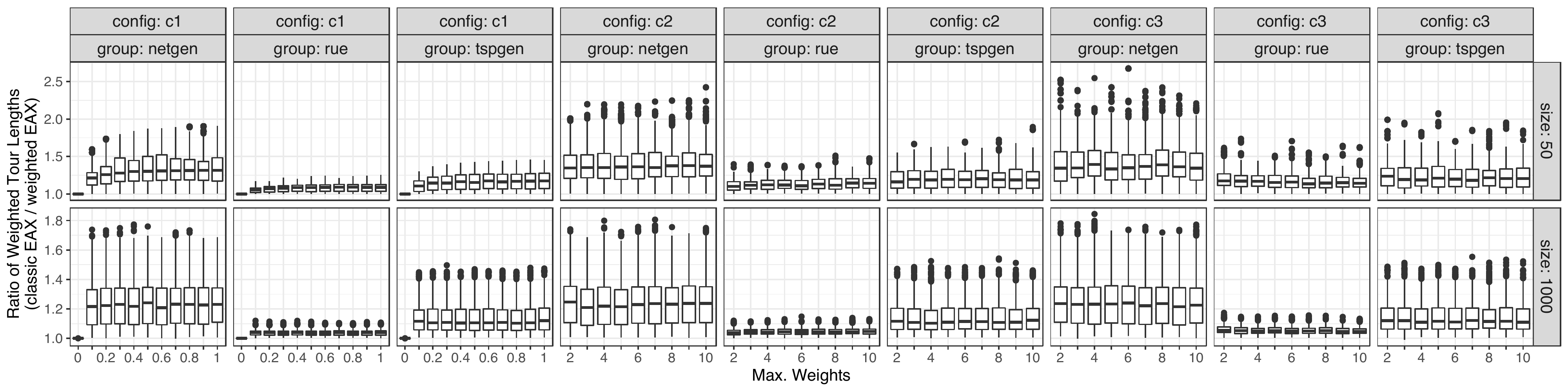}
    \caption{Boxplots illustrating ratios between found tour lengths. Each instance was optimized with EAX -- which internally used the weighted or classical TSP fitness function, respectively -- and all resulting tours have been assessed using the weighted fitness. We show results for $n \in \{50, 1000\}$ due to space limitations, but patterns for omited data are similar.}
    \label{fig:boxplot_ratios}
\end{figure*}

\begin{figure*}[tb]
    \centering
    \includegraphics[width=\textwidth]{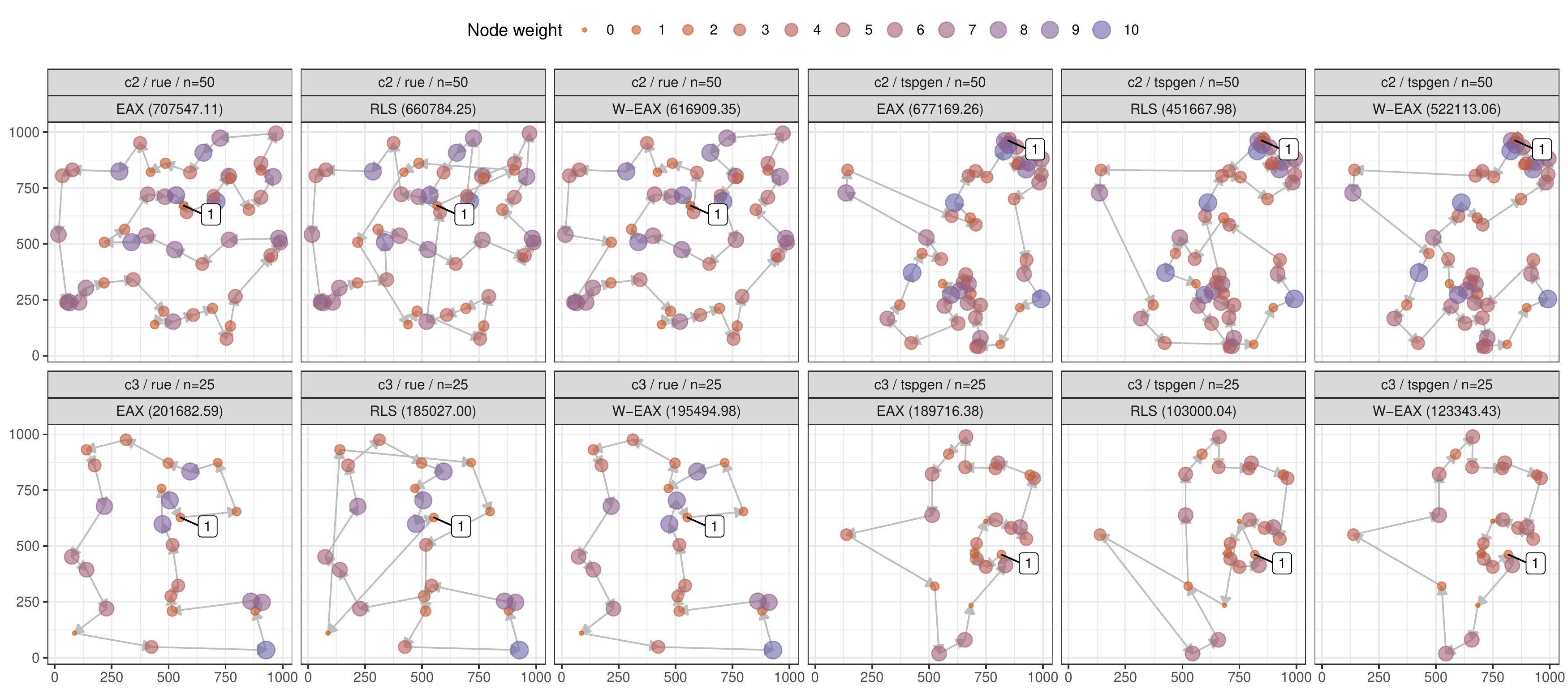}
    \caption{This plot shows \wdtsp{} tours determined by EAX, W-EAX and RLS, respectively, for four selected, yet representative problem instances. The start node $\pi_1 = 1$ is highlighted and the direction of the tour is indicated by arrows.}
    \label{fig:tours}
\end{figure*}

As expected we observed all ratios being greater than 1 with median values at about 1.15 across all combinations. Frequently, large outliers reached ratios up to 1.75 for instances with at least 100 nodes and even $> 2.5$ for smaller $n$. However, in general, no patterns can be identified with respect to configuration or maximal node weight. The sole exception is configuration $C_1$ and $n \in \{25, 50\}$ where we observe a slightly increasing trend in median ratios with increasing $d \in \{0.0, 0.1, \ldots, 1.0\}$ (see top-left boxplots in Fig.~\ref{fig:boxplot_ratios}). This trend vanishes for $n \geq 100$ as a high number of nodes already imposes a large cumulative weight when considering only a part of any tour.

In contrast, comparing node placement, we observe strong differences. While the ratios are lowest when the nodes are placed uniformly at random (rue), with values below $1.15$ for large instances with at least $500$ nodes, ratios become increasingly larger with increasing instance structure. For $n = 1\,000$ nodes, ratios on tspgen instances go up to about $1.5$ while on netgen instances -- with strongly segregated clusters -- ratios reach values up to $1.8$ with a median of about $1.2$. Hence, more than $50\%$ of the ratios are higher than the maximal ratio in case of rue placement.

Fig.~\ref{fig:tours} shows exemplary tours obtained by EAX, W-EAX and -- for comparison -- RLS with inversion mutation. The tours are shown for four instances of class $C_2$ (top row) and class $C_3$ (bottom row). We can make the following observation: since the variation operator of EAX was not modified, the resulting tours of both EAX and W-EAX are free of crossings. However, for the \wdtsp{}, optimal tours do not necessarily need to avoid crossings. As for instance shown in the fifth column of Fig.~\ref{fig:tours} (top row), RLS often finds solutions with many crossings, resulting in much shorter tours than produced by both EAX variants.
%than the tours calculated by both EAX variants.

Additionally, for the \wdtsp{} -- in particular in the presence of segregated cluster structures -- it is often beneficial that long edges are included early in the permutation, as a later consideration would be associated with (the burden of) a huge amount of accumulated node weights. Once again, this is observable in the fifth column of Fig.~\ref{fig:tours}: RLS places long edges early in the tour to quickly reach the top left cluster (top row) or the isolated node (bottom row), and then leave it just as quickly again. In case of netgen instances an even stronger effect of long inter-cluster edges can be expected. Here, W-EAX manages -- due to the \wdtsp{} fitness function -- to cumulate less weight in the respective clusters before they are left in order to reach another cluster. This explains why ratios increase with increasing cluster segregation.

\section{Conclusions}
Motivated by different complex variants of the traveling salesperson problem, we have introduced the node weight dependent TSP called \wdtsps{} which captures aspects of important complex TSP variants such as the time dependent TSP or the traveling thief problem. We have pointed out the relation of \wdtsps{} to the TSP and how the weights on the nodes impact the structure of the problem. Our insights provided the tools for designing approximation algorithms for the metric version of the problem. Furthermore, we have shown that approximation algorithms for the $\{1,2\}$-TSP can be used as the basis for approximation algorithms for \wdtsps{} when also considering the reverse tour.

Our experimental studies show that, on almost all considered settings and a wide range of instances, inversion mutation is superior to exchange and jump operators when adopted by randomized local search. Furthermore, experimental investigations on the state-of-the-art TSP solver EAX show that the weights lead to significantly different results when comparing \wdtsps{} to the classical TSP. For $n=50$ nodes, we encountered instances where the performances of the obtained tours differ by a factor of $2.75$ whereas this factor is $1.75$ for a wide range of instances having $n=1000$ cities.

% Our experimental investigations using EAX on a wide range of instances show that the weights lead to a significantly different results when comparing \wdtsps{} to the classical TSP. For $n=50$ nodes, we encountered instances where the performance of the obtained tours differs by a factor of $2.75$ whereas this factor is $1.75$ for a wide range of instances having $n=1000$ cities.

%\frank{Need to add something on experiments}

For future work, it would be interesting to study other state-of-the-art heuristics for the TSP and how to adapt them to \wdtsps{}. In order to systematically judge the performance of such approaches, it would be highly beneficial to have efficient exact solvers for \wdtsps{}.

% \begin{acks}
% ...
% \end{acks}

\bibliographystyle{plain}
\bibliography{arxiv.bib}

\begin{thebibliography}{10}

\bibitem{tdtsp_newer_ilp}
Louis{-}Philippe Bigras, Michel Gamache, and Gilles Savard.
\newblock {The Time-Dependent Traveling Salesman Problem and Single Machine
  Scheduling Problems with Sequence Dependent Setup Times}.
\newblock {\em {Discrete Optimization}}, 5(4):685~--~699, 2008.

\bibitem{Blum_latency}
Avrim Blum, Prasad Chalasani, Don Coppersmith, William~R. Pulleyblank,
  Prabhakar Raghavan, and Madhu Sudan.
\newblock {The Minimum Latency Problem}.
\newblock In {\em {Proceedings of the Twenty-Sixth Annual {ACM} Symposium on
  Theory of Computing}}, pages 163~--~171, May 1994.

\bibitem{DBLP:conf/cec/BonyadiMB13}
Mohammad~Reza Bonyadi, Zbigniew Michalewicz, and Luigi Barone.
\newblock {The Travelling Thief Problem: The First Step in the Transition from
  Theoretical Problems to Realistic Problems}.
\newblock In {\em {Proceedings of the IEEE Congress on Evolutionary Computation
  (CEC)}}, pages 1037~--~1044. IEEE, June 2013.

\bibitem{DBLP:books/sp/19/BonyadiM0019}
Mohammad~Reza Bonyadi, Zbigniew Michalewicz, Markus Wagner, and Frank Neumann.
\newblock {Evolutionary Computation for Multicomponent Problems: Opportunities
  and Future Directions}.
\newblock In Shubhabrata Datta and J.~Paulo Davim, editors, {\em {Optimization
  in Industry, Present Practices and Future Scopes}}, pages 13~--~30. Springer,
  2019.

\bibitem{BKNW+19}
Jakob Bossek, Pascal Kerschke, Aneta Neumann, Markus Wagner, Frank Neumann, and
  Heike Trautmann.
\newblock {Evolving Diverse TSP Instances by Means of Novel and Creative
  Mutation Operators}.
\newblock In {\em {Proceedings Foundations of Genetic Algorithms (FOGA)}},
  pages 58~--~71. {ACM Press}, 2019.

\bibitem{1-2-approx}
Bj{\"{o}}rn Brod{\'{e}}n, Mikael Hammar, and Bengt~J. Nilsson.
\newblock {Online and Offline Algorithms for the Time-Dependent TSP with Time
  Zones}.
\newblock {\em Algorithmica}, 39(4):299~--~319, 2004.

\bibitem{latency_best_apx}
Kamalika Chaudhuri, Brighten Godfrey, Satish Rao, and Kunal Talwar.
\newblock {Paths, Trees, and Minimum Latency Tours}.
\newblock In {\em {Proceedings of the 44th Symposium on Foundations of Computer
  Science (FOCS)}}, pages 36~--~45, October 2003.

\bibitem{christofides}
Nicos Christofides.
\newblock Worst-{Case} {Analysis} of a {New} {Heuristic} for the {Travelling}
  {Salesman} {Problem}.
\newblock Technical Report 388, Graduate School of Industrial Administration,
  Carnegie Mellon University, 1976.

\bibitem{graphical_tsp}
G{\'{e}}rard Cornu{\'{e}}jols, Jean Fonlupt, and Denis Naddef.
\newblock {The Traveling Salesman Problem on a Graph and Some Related
  Integerpolyhedra}.
\newblock {\em Mathematical Programming}, 33(1):1~--~27, 1985.

\bibitem{ElYafrani:2016:PVS:2908812.2908847}
Mohamed El~Yafrani and Bela\"{\i}d Ahiod.
\newblock {Population-based vs. Single-solution Heuristics for the Travelling
  Thief Problem}.
\newblock In {\em {Genetic and Evolutionary Computation Conference (GECCO)}},
  pages 317~--~324. ACM, 2016.

\bibitem{Faulkner2015}
Hayden Faulkner, Sergey Polyakovskiy, Tom Schultz, and Markus Wagner.
\newblock {Approximate Approaches to the Traveling Thief Problem}.
\newblock In {\em {Conference on Genetic and Evolutionary Computation
  (GECCO)}}, pages 385~--~392. ACM, 2015.

\bibitem{repairmen}
Raja Jothi and Balaji Raghavachari.
\newblock {Approximating the k-Traveling Repairman Problem with Repairtimes}.
\newblock {\em {Journal of Discrete Algorithms}}, 5(2):293~--~303, 2007.

\bibitem{KKBHTLeveragingTSP}
Pascal Kerschke, Lars Kotthoff, Jakob Bossek, Holger~H Hoos, and Heike
  Trautmann.
\newblock {Leveraging TSP Solver Complementarity through Machine Learning}.
\newblock {\em Evolutionary Computation}, 26(4):597~--~620, 2018.

\bibitem{Kotthoff2015}
Lars Kotthoff, Pascal Kerschke, Holger~H. Hoos, and Heike Trautmann.
\newblock {Improving the State of the Art in Inexact TSP Solving Using
  Per-Instance Algorithm Selection}.
\newblock In Clarisse Dhaenens, Laetitia Jourdan, and Marie-El{\'e}onore
  Marmion, editors, {\em {Proceedings of the 9th International Conference on
  Learning and Intelligent Optimization (LION)}}, volume 8994 of {\em {Lecture
  Notes in Computer Science (LNCS)}}, pages 202~--~217. Springer, January 2015.

\bibitem{wrong_tdtsp}
Chryssi Malandraki and Mark~S. Daskin.
\newblock {Time Dependent Vehicle Routing Problems: Formulations, Properties
  and Heuristic Algorithms}.
\newblock {\em Transportation Science}, 26(3):185~--~200, 1992.

\bibitem{DBLP:journals/soco/MeiLY16}
Yi~Mei, Xiaodong Li, and Xin Yao.
\newblock {On Investigation of Interdependence Between Sub-Problems of the
  Travelling Thief Problem}.
\newblock {\em Soft Computing}, 20(1):157~--~172, 2016.

\bibitem{Nagata2013}
Yuichi Nagata and Shigenobu Kobayashi.
\newblock {A Powerful Genetic Algorithm Using Edge Assembly Crossover for the
  Traveling Salesman Problem}.
\newblock {\em {INFORMS Journal on Computing}}, 25(2):346~--~363, 2013.

\bibitem{DBLP:conf/algocloud/NeumannPSSW18}
Frank Neumann, Sergey Polyakovskiy, Martin Skutella, Leen Stougie, and Junhua
  Wu.
\newblock {A Fully Polynomial Time Approximation Scheme for Packing While
  Traveling}.
\newblock In {\em {4th International Symposium on Algorithmic Aspects of Cloud
  Computing (ALGOCLOUD), Revised Selected Papers}}, volume 11409 of {\em LNCS},
  pages 59~--~72. Springer, August 2018.

\bibitem{DBLP:journals/mor/PapadimitriouY93}
Christos~H. Papadimitriou and Mihalis Yannakakis.
\newblock The traveling salesman problem with distances one and two.
\newblock {\em Math. Oper. Res.}, 18(1):1--11, 1993.

\bibitem{tdtsp_origin}
Jean{-}Claude Picard and Maurice Queyranne.
\newblock {The Time-Dependent Traveling Salesman Problem and Its Application to
  the Tardiness Problem in One-Machine Scheduling}.
\newblock {\em {Operations Research}}, 26(1):86~--~110, 1978.

\bibitem{DBLP:conf/gecco/PolyakovskiyB0MN14}
Sergey Polyakovskiy, Mohammad~Reza Bonyadi, Markus Wagner, Zbigniew
  Michalewicz, and Frank Neumann.
\newblock {A Comprehensive Benchmark Set and Heuristics for the Traveling Thief
  Problem}.
\newblock In {\em {Proceedings of the Genetic and Evolutionary Computation
  Conference (GECCO)}}, pages 477~--~484. {ACM}, July 2014.

\bibitem{tdtsp_genetic}
Leonard~J. Testa, Albert~C. Esterline, Gerry~V. Dozier, and Abdollah Homaifar.
\newblock {A Comparison of Operators for Solving Time dependent Traveling
  Salesman Problems Using Genetic Algorithms}.
\newblock In {\em {Proceedings of the Genetic and Evolutionary Computation
  Conference (GECCO)}}, pages 995--1102, July 2000.

\bibitem{Wagner2017ttpalgssel}
Markus Wagner, Marius Lindauer, Mustafa M{\i}s{\i}r, Samadhi Nallaperuma, and
  Frank Hutter.
\newblock {A Case Study of Algorithm Selection for the Traveling Thief
  Problem}.
\newblock {\em Journal of Heuristics}, 24(3):295~--~320, Apr 2017.

\bibitem{DBLP:conf/gecco/WuPN16}
Junhua Wu, Sergey Polyakovskiy, and Frank Neumann.
\newblock {On the Impact of the Renting Rate for the Unconstrained Nonlinear
  Knapsack Problem}.
\newblock In {\em {Proceedings of the Genetic and Evolutionary Computation
  Conference (GECCO)}}, pages 413~--~419. {ACM}, July 2016.

\bibitem{DBLP:conf/gecco/WuP0N18}
Junhua Wu, Sergey Polyakovskiy, Markus Wagner, and Frank Neumann.
\newblock {Evolutionary Computation Plus Dynamic Programming for the
  Bi-Objective Travelling Thief Problem}.
\newblock In {\em {Proceedings of the Genetic and Evolutionary Computation
  Conference (GECCO)}}, pages 777~--~784. {ACM}, July 2018.

\bibitem{DBLP:conf/seal/Wu0PN17}
Junhua Wu, Markus Wagner, Sergey Polyakovskiy, and Frank Neumann.
\newblock {Exact Approaches for the Travelling Thief Problem}.
\newblock In {\em {Proceedings of the 11th International Conference on
  Simulated Evolution and Learning (SEAL)}}, pages 110~--~121, November 2017.

\bibitem{ElYafrani2018231}
Mohamed~El Yafrani and Bela\"{\i}d Ahiod.
\newblock {Efficiently Solving the Traveling Thief Problem Using Hill Climbing
  and Simulated Annealing}.
\newblock {\em {Information Sciences}}, 432:231~--~244, 2018.

\end{thebibliography}

\end{document}